%% file: lipics-main.tex
\documentclass[a4paper,UKenglish,cleveref,autoref,final]{lipics-v2021}

\usepackage[disable]{todonotes}

\hideLIPIcs  

\input{macros-lipics.tex}
\usepackage{cleveref}
\usepackage{pgfplots}
\pgfplotsset{compat=1.18}

\title{On Counting Graphs for Modal Logic with Counting} 

\author{Avijeet {Ghosh}}{MC2, LIP, ENS de Lyon, France \and \url{https://itsmeavi.github.io/} }{avi.ghosh23@gmail.com}{}{}

\author{Fran{\c{c}}ois Schwarzentruber}{MC2, LIP, ENS de Lyon, France \and \url{https://perso.ens-lyon.fr/francois.schwarzentruber/}}{francois.schwarzentruber@ens-lyon.fr}{}{}

\authorrunning{A. Ghosh and F. Schwarzentruber} 

\Copyright{Avijeet Ghosh and Fran{\c{c}}ois Schwarzentruber} 



\category{} 

\relatedversion{} 

\nolinenumbers 

\EventEditors{John Q. Open and Joan R. Access}
\EventNoEds{2}
\EventLongTitle{42nd Conference on Very Important Topics (CVIT 2016)}
\EventShortTitle{CVIT 2016}
\EventAcronym{CVIT}
\EventYear{2016}
\EventDate{December 24--27, 2016}
\EventLocation{Little Whinging, United Kingdom}
\EventLogo{}
\SeriesVolume{42}
\ArticleNo{23}

\begin{document}

\maketitle

\begin{abstract}
In this paper, we study the theoretical computational complexity of counting trees and graphs that are classified positively by a graph neural networks, or said equivalently that satisfy a given formula of modal logic with counting. We also study convergence. 
\end{abstract}

\newpage
\section{Introduction}
\label{section:introduction}
\input{section_introduction.tex}

\section{Related Work}
\label{section:relatedwork} 
\input{section_relatedwork.tex}

\section{Background and Preliminaries}
\label{section:background}
\input{section_background}

\section{Complexity of Counting models}
\label{section:countingcomplexity}
\input{section_counting.tex}

\section{Theoretical results at infinity}
\label{section:convergence}
\input{section_convergence.tex}

\section{Open problems}
\label{section:conclusion}
\input{section_conclusion.tex}

\bibliography{biblio}

\newpage
\appendix

\input{appendix.tex}

\end{document}

%% file: macros-lipics.tex
\usepackage{ifthen}
\usepackage{adjustbox}
\usepackage{comment}
\usepackage{multicol}
\usepackage{blkarray}
\usepackage{bussproofs}
\usepackage{tikz}
\usetikzlibrary{automata}
\usepackage{xcolor}
\usepackage{graphicx}
\usepackage{amsfonts}
\usepackage{amsmath}
\usepackage{amsthm}
\usepackage{hyperref}
\usepackage[ruled,vlined,linesnumbered]{algorithm2e}
\usepackage{algorithmic}
\usepackage{thmtools}

\usetikzlibrary{decorations.markings}
\usetikzlibrary{calc,patterns,angles,quotes,shapes.geometric,positioning}

\newcommand\indexalert\emph

\newcommand{\AP}{\ensuremath{\mathcal{P}}}

\newcommand{\M}{\mathcal{M}}

\newcommand{\Sharp}{\#}

\protected\def\MSO{\ifmmode \mbox{\sc MSO} \else {\sc MSO}\xspace\fi}
\protected\def\FO{\ifmmode \mbox{\sc FO} \else {\sc FO}\xspace\fi}
\protected\def\MMSO{\ifmmode {\mbox{\sc M}\MSO} \else {{\sc M}\MSO}\xspace\fi}
\protected\def\MFO{\ifmmode \mbox{\sc MFO} \else {\sc MFO}\xspace\fi}

\newcommand{\union}{\cup}

\newcommand{\set}[1]{\left\{#1\right\}}

\newcommand{\suchthat}{\mid}

\newcommand{\bigand}{\bigwedge}
\newcommand{\lbigand}{\bigand}

\newcommand{\bigor}{\bigvee}
\newcommand{\lbigor}{\bigor}

\newcommand{\limply}{\rightarrow}

\renewcommand{\phi}{\varphi}

\newcommand{\lbox}{\square}

\newcommand{\ldiamond}{\Diamond}
\newcommand{\tuple}[1]{\langle #1 \rangle}

\newcommand{\logdiamarg}[1]{\langle #1 \rangle}
\newcommand{\ldiaarg}[1]{\logdiamarg{#1}}

\newcommand{\setN}{\mathbb{N}}

\newcommand{\ensN}{\setN}

\newtheorem{fact}[theorem]{Fact}

\newcommand{\algofunction}{\textbf{function }}

\newcommand{\algoprocedure}{\textbf{procedure }}

\newcommand{\algofor}{\textbf{for }}
\newcommand{\algodo}{\textbf{do }}

\definecolor{algocommentbackgroundcolor}{rgb}{1,1,0.5}

\newcommand{\algowhile}{\textbf{while }}

\newcommand{\algoif}{\textbf{if }}
\newcommand{\algothen}{\textbf{then }}

\newcommand{\algomatch}{\textbf{match }}

\newcommand{\algocase}{\textbf{case }}

\newcommand{\algoaccept}{\textbf{accept }}
\newcommand{\algoreject}{\textbf{reject }}

\newlength{\algoindentlongueur}
\newcommand{\algoindent}{\hspace*{\algoindentlongueur}}

\newlength{\algoindentavantvrulelongueur}
\newcommand{\algoindentavantvrule}{\hspace*{\algoindentavantvrulelongueur}}

\newlength{\dummy}

\newsavebox{\frameminipageboiteavecunnomsuperlongdesortequonnepuissepaslereutiliser}
\newenvironment{frameminipage}[2][c]{%
\begin{lrbox}{\frameminipageboiteavecunnomsuperlongdesortequonnepuissepaslereutiliser}%
\begin{minipage}[#1]{#2}%
} {%
\end{minipage}%
\end{lrbox}%
\framebox{\usebox{\frameminipageboiteavecunnomsuperlongdesortequonnepuissepaslereutiliser}}%
}

\newenvironment{algobloc}{\setlength{\dummy}{\linewidth}\addtolength{\dummy}{- \algoindentlongueur}\addtolength{\dummy}{- \algoindentavantvrulelongueur}\algoindentavantvrule\vrule\algoindent\begin{minipage}{\dummy}}{\end{minipage}}

\tikzstyle{zoneprogramcounter} = [fill=gray!40, draw=none]
\tikzstyle{zonedatacounter} = [fill=gray!40, draw=none, decorate, decoration={snake, amplitude=1pt, segment length=4pt}]
\tikzstyle{portion} = [densely dotted, shade, top color = white, bottom color = gray!40]
\tikzstyle{portiondata} = [densely dotted, shade, top color = white, bottom color = gray!40, decorate, decoration={snake, amplitude=1pt, segment length=4pt}]
\tikzset{
	double arrow/.style args={#1 colored by #2 and #3}{
		-stealth,line width=#1,#2, 
		postaction={draw,-triangle 90 cap,#3,line width=(#1)-4pt,
			shorten <=4pt,shorten >=4}, 
	}
}
\tikzstyle{tapearrow} = [double arrow=10pt colored by black!50!white and black!20!white, -triangle 90 cap, fill = white]%
\tikzstyle{execarrow} = [line width=1pt,->]%

\tikzstyle{world}=[inner sep=0.5mm]
\tikzstyle{event}=[fill=gray!30, inner sep=0.5mm]
\tikzstyle{realworldarrowfromleft} = [initial left, initial text={}]
\tikzstyle{realworldarrowfromright} = [initial right, initial text={}]

\protected\def\ZPP{\ifmmode \mbox{\sc ZPP} \else {\sc ZPP}\xspace\fi}
\protected\def\RP{\ifmmode \mbox{\sc RP} \else {\sc RP}\xspace\fi}
\protected\def\coRP{\ifmmode \mbox{\sc coRP} \else {\sc coRP}\xspace\fi}
\protected\def\DTIME{\ifmmode \mbox{\sc Dtime} \else {\sc Dtime}\xspace\fi}
\protected\def\NTIME{\ifmmode \mbox{\sc Ntime} \else {\sc Ntime}\xspace\fi}
\protected\def\DSPACE{\ifmmode \mbox{\sc Dspace} \else {\sc Dspace}\xspace\fi}
\protected\def\NSPACE{\ifmmode \mbox{\sc Nspace} \else {\sc Nspace}\xspace\fi}
\protected\def\NP{\ifmmode \mbox{\sc NP} \else {\sc NP}\xspace\fi}
\protected\def\coNP{\ifmmode \mbox{\sc coNP} \else {\sc coNP}\xspace\fi}
\protected\def\NPSPACE{\ifmmode \mbox{\sc NPspace} \else {\sc NPspace}\xspace\fi}
\protected\def\PSPACE{\ifmmode \mbox{\sc Pspace} \else {\sc Pspace}\xspace\fi}
\protected\def\EXPSPACE{\ifmmode \mbox{\sc Expspace} \else {\sc Expspace}\xspace\fi}
\protected\def\TWOEXPSPACE{\ifmmode \mbox{\sc 2Expspace} \else {\sc 2Expspace}\xspace\fi}
\protected\def\PTIME{\ifmmode \mbox{\sc P} \else {\sc P}\xspace\fi}
\protected\def\NPTIME{\ifmmode \mbox{\sc NP} \else {\sc NP}\xspace\fi}
\protected\def\EXPTIME{\ifmmode \mbox{\sc Exptime} \else {\sc Exptime}\xspace\fi}
\protected\def\AEXPTIME{\ifmmode \mbox{\sc Aexptime} \else {\sc Aexptime}\xspace\fi}
\protected\def\NEXPTIME{\ifmmode \mbox{\sc NExptime} \else {\sc NExptime}\xspace\fi}
\protected\def\2EXPTIME{\ifmmode \mbox{\sc 2-Exptime} \else {\sc
		2-Exptime}\xspace\fi}
\DeclareRobustCommand{\kEXPTIME}[1][k]{\ifmmode \mbox{\sc $#1$-Exptime}
	\else {\sc $#1$-Exptime}\xspace\fi}
\DeclareRobustCommand{\kNEXPTIME}[1][k]{\ifmmode \mbox{\sc $#1$-NExptime}
	\else {\sc $#1$-NExptime}\xspace\fi}
\DeclareRobustCommand{\kEXPSPACE}[1][k]{\ifmmode \mbox{\sc $#1$-Expspace}
	\else {\sc $#1$-Expspace}\xspace\fi}
\protected\def\ELEMENTARY{\ifmmode \mbox{\sc Elementary} \else {\sc Elementary}\xspace\fi}
\protected\def\AEXPpol{\ifmmode \mbox{{\sc A}_{\text{pol}}\EXPTIME} \else
	{\sc A}$_{\text{pol}}$\EXPTIME\fi}
\protected\def\APTIME{\ifmmode \mbox{\sc Aptime} \else {\sc Aptime}\xspace\fi}
\protected\def\AEXPSPACE{\ifmmode \mbox{\sc Aexpspace} \else {\sc Aexpspace}\xspace\fi}

\tikzstyle{cell} = [draw,minimum height=5mm,minimum width=5mm]
\tikzset{
	cellcolor/.cd,
	0/.style={fill=gray!20!white},
	1/.style={fill=yellow!20!white}
}
\tikzstyle{cellalive} = [fill=yellow!20!white]

\newcommand\listeventsempty\epsilon

\newcommand{\initialword}{\omega}

\definecolor{C}{rgb}{1,1,1}
\definecolor{C }{rgb}{0.8,0.8,0.8}

\definecolor{Ca}{rgb}{0.95,1,0.5}
\definecolor{Cb}{rgb}{1,0.9,0.5}
\definecolor{Cw_2}{rgb}{0.95,1,0.5}
\definecolor{Cw_n}{rgb}{1,0.9,0.5}

\definecolor{Cq'-}{rgb}{0.5,0.5,1}
\definecolor{C-q_0}{rgb}{0.5,0.5,1}
\definecolor{C-q_1}{rgb}{0.5,0.4,1}
\definecolor{C-q'}{rgb}{0.5,0.8,0.5}
\definecolor{Cq_f-}{rgb}{0.4,0.9,0.5}

\definecolor{Cq_0,w_1}{rgb}{1,0.6,0.5}
\definecolor{Cq_0,}{rgb}{1,0.6,0.4}
\definecolor{Cq_2,}{rgb}{0.8,0.9,0.4}
\definecolor{Cq_2,a}{rgb}{0.8,0.9,0.4}
\definecolor{Cq_f,a}{rgb}{0.7,0.9,0.4}
\definecolor{Cq_0,a}{rgb}{1,0.6,0.5}
\definecolor{Cq_0,b}{rgb}{1,0.5,0.5}
\definecolor{Cq_1,a}{rgb}{1,0.5,0.4}
\definecolor{Cq_1,b}{rgb}{1,0.4,0.4}
\definecolor{Cq,a}{rgb}{1,0.5,0.5}
\definecolor{Cq',a}{rgb}{1,0.5,0.5}
\definecolor{Cq',b}{rgb}{1,0.5,0.5}
\definecolor{Cq_0,1}{rgb}{0.5,0.5,0.5}

\provideboolean{student}

\newcommand{\trou}[1]{\ifthenelse{\boolean{student}}
	{\colorbox{gray!10!white}{\phantom{#1}}}
	{#1}
}

\ifthenelse{\boolean{student}}
{\excludecomment{hidden}}
{}

\usepackage{mdframed}

\newmdenv{allfour}

\newmdenv[rightline=false,topline=false,bottomline=false]{solutionbox}

\ifthenelse{\boolean{student}}
{\excludecomment{solution}}
{
}

\ifthenelse{\boolean{student}}
{\excludecomment{examplehidden}}
{}

\ifthenelse{\boolean{student}}
{\excludecomment{exercicehidden}}
{}

\ifthenelse{\boolean{student}}
{\excludecomment{remarkhidden}}
{}

\definecolor{proba}{RGB}{0, 0, 128}

\newcommand\proba[1]{
\colorlet{savedleftcolor}{.}
	\mathbb P
	(
    \color{savedleftcolor}
	#1\color{proba}
	)
    \color{savedleftcolor}
	}

\usepackage{tikzsymbols}

\newcommand{\taille}[1]{|#1|}
\newcommand\sizeof\taille

\newcommand{\convergesn}{\xrightarrow[]{n \rightarrow + \infty}}

\newcommand\instancemot\initialword

\newcommand\Ksharp{\ensuremath{K{^\#}}}

\newcommand\DSF{\ensuremath{\mathsf{Dsf}}}

\newcommand{\CountingTree}{\ensuremath{\textsc{Counting{-}Tree}}\xspace}
\newcommand{\guesswithbias}{\ensuremath{\mathsf{guessGraph}}\xspace}
\newcommand{\CountingGraph}{\ensuremath{\textsc{Counting{-}Graph}}\xspace}

\newcommand{\guesstree}{\ensuremath{\mathsf{guessTree}}\xspace}

\newcommand{\explorealltrees}{\ensuremath{\mathsf{exploreAllTrees}}\xspace}
\newcommand{\sem}{\ensuremath{\mathsf{counter}}}

%% file: section_introduction.tex
Validity problems in modal logic is relevant in many reasoning tasks: verification of systems, epistemic planning.
However, validity problem is rigid in this sense: a formula is valid if it holds in \emph{all} models. In this paper, we relax the validity problem to model counting and we study the \emph{probabilistic validity problem}:
Is a formula true in a \emph{large proportion} of models?
Is a formula asymptotically true for large models?
What is the proportion of models in which a formula holds?
First, let us give motivation from a verification problem of graph neural networks (GNNs).
\newcommand\phiGNN{\phi_{\text{GNN}}}

Recently, an exact correspondence between modal logic with counting verifying GNNs has been settled and applied for verifying properties on GNNs 
(\cite{DBLP:journals/corr/abs-2307-05150}, \cite{DBLP:conf/ijcai/NunnSST24}, \cite{DBLP:conf/icalp/BenediktLMT24}).
In this paper, the only thing the reader must know is that each GNN from a large class of GNNs, namely with an eventually constant activation function or let say with truncated ReLU, is equivalent to a formula $\phiGNN$ from a logic called $\Ksharp$, which is a natural extension of graded modal logic with linear arithmetic inequations (see \cite{DBLP:journals/corr/abs-2307-05150}, \cite{DBLP:conf/ijcai/NunnSST24}, \cite{DBLP:conf/icalp/BenediktLMT24}). The notion of equivalence is: a pointed graph $G, u$ is classified positively by the GNN iff $\phiGNN$ holds in $G, u$.

An example of $\Ksharp$-formula is $\# p \geq \# q + 2$ which means that there are more $p$-successors than $q$-successors plus two. For instance, the formula $\# p \geq \# q + 2$ is true in the pointed graph 
\tikzstyle{vertex} = [draw,circle, minimum height=5mm, inner sep=0.5mm]
\begin{tikzpicture}[thick, baseline=0mm]
    \node[vertex, initial, initial text={}] (root) at (0, 1) {};
    \node[vertex] (s1) at (-3, 0) {$p$};
    \node[vertex] (s2) at (-2, 0) {$p$};
    \node[vertex] (s3) at (-1, 0) {$p$};
    \node[vertex] (s4) at (0, 0) {$q$};
    \node[vertex] (s5) at (1, 0) {};
    \node[vertex] (s6) at (2, 0) {};
    \node[vertex] (s7) at (3, 0) {};
    \foreach \x in {1, 2, 3, 4, 5, 6, 7} {
        \draw[->] (root) -- (s\x);
    }
\end{tikzpicture}
since $3 \geq 1 + 2$.

Verifying GNNs is important in order to understand how they classify data. For instance, we could verify that a GNN recommends a person whenever that person has at least 3 violinist friends, meaning that for \emph{all} pointed graphs $(G, u)$, if the current vertex has at least 3 violinist successors then the GNN classifies $(G, u)$ positively. This kind of verification tasks were addressed in \cite{DBLP:conf/ijcai/NunnSST24}, and on quantized GNNs in \cite{DBLP:conf/ijcai/SalzerST25} and \cite{DBLP:conf/kr/ChernobrovkinSS26}.

However, it is unlikely that a property holds for \emph{all} GNNs. For instance, it is unlikely that a GNN recommends a person whenever that person has at least 3 violinist friends. 
The reason is because the concept of recommended persons is a vague concept that is difficult to explain and does not completely match a property as precise as "having at least 3 violinist friends" (written $\# violinist \geq 3$). The second reason is that there are approximations due to learning process itself, making unlikely that the obtained GNN exactly corresponds to a given property.

However, it is possible that a property like "having at least 3 violinist friends" is a probable explanation of the classification of the GNN.
That is why, instead of checking that $\phi \limply \phiGNN$ is valid, we want to compute the probability that $\phi \limply \phiGNN$ holds where $\phiGNN$ is a formula that is true when the GNN recommends the current model, and $\phi$ is a property, e.g.  ``the person has at least 3 violinist friends", $\# violinist \geq 3$. We may then conclude that the probability that $\phi \limply \phiGNN$ holds in graph of size 42 is more than 0.8. 

Our paper is related to 0-1 law that first-order logic has: any FO-sentence $\phi$ is asymptotically true in large models or false in large models. While specific classes of GNNs have 0-1 law (\cite{DBLP:conf/nips/Adam-DayIC23}, \cite{NEURIPS2024_e1cd1db6}), this is clearly false for logic $\Ksharp$. Indeed, the limit of the truth of $\# p \geq \# q$ is 1/2. This justify the study of computing the probability of a formula $\phi$ given the size of the graph, or counting the number of $\phi$-models of size $n$. As many modal logics have the  tree-model property, we also focus on trees and counting the number of trees satisfying $\phi$ given a branching factor.


\begin{itemize}
    \item \emph{Counting models}. We study the complexity of counting $\phi$-trees/graphs, given formula $\phi$ from logic $\Ksharp$ (see \Cref{figure:complexityresults}).
    \item \emph{Convergence}. We show that the proportion of $\phi$-trees does not converge for some $\phi$. We propose an artificial trick to make it converge.
\end{itemize}


\begin{figure}[t]
    \hspace*{0pt}
 	\begin{center}
 \begin{tabular}{l|l|l}
  & $n$ unary & $n$ binary \\
  \hline
Trees & \multicolumn{2}{|c}{$\Sharp\PSPACE$} \\
 & \multicolumn{2}{|c}{$\PSPACE$-complete}  \\
\hline
Graphs & $\Sharp \PTIME$-complete & $\Sharp\EXPTIME$ \\
& & \PSPACE-hard 
 \end{tabular}\label{table:countingresults}
 \end{center}
\caption{Complexity results about counting trees ($n$ is the branching factor) or graphs ($n$ is then the number of vertices).\label{figure:complexityresults}}
\end{figure}

\textbf{Outline. }
In \Cref{section:relatedwork} we discuss the related work.
In \Cref{section:background}, we recall the formal definition of logic $\Ksharp$.
In \Cref{section:countingproblem}, we define the counting problems we are studying.
In \Cref{section:countingcomplexity}, we provide our complexity results. We also have an implementation using Monte-carlo method that samples arbitrary large graphs and trees and gives out the proportion of satisfiability\footnote{For anonymity reasons, we do not provide the code/repository link now.}.
In \Cref{section:convergence}, we exemplify non-convergence in trees and provide a patch for having convergence.
In \Cref{section:conclusion}, we conclude.
\newline
\noindent\fbox{
\parbox{0.97\linewidth}{
Some parts of the
 proof of~\Cref{theorem:muprimeconverge} 
 were obtained by AI, although the authors have come up with the general approach by themselves.
 }
 }



%% file: section_relatedwork.tex
The study of convergence of properties when the size of the models goes to infinity goes back to 0-1 law of first-order logic (\cite{glebskii1969range}, \cite{DBLP:journals/jsyml/Fagin76}). It means that the proportion of models satisfying a first-order property $\phi$ converges to 0 or 1 when the size $n$ of the model goes to infinity. In this work, the notion of random graph is given by the Erdős–Rényi random graph model.

Later, Granjean proved that deciding whether the limit is 1 for a given FO property is PSPACE-complete~\cite{DBLP:journals/iandc/Grandjean83}.
As modal logic is a syntactic fragment of first-order logic, it also has 0-1 law. But, the computation of the limit of a modal formula is $\Delta^p_2$-complete when the formula is represented by a DAG, and $\Delta^{p, \log n}_2$-complete if represented by a tree (see Theorem 4.6 in \cite{DBLP:journals/apal/HalpernK94}), where $\Delta^p_2 = \text{\PTIME}^\text{\NP}$ is the class of problems decided in poly-time with an NP oracle, and $\Delta^{p, \log n}_2$ is the class of problems decided in poly-time with at most $O(\log n)$ queries to some NP oracle.

First-order logic when there are function symbols does not have the 0-1 law but still the proportion of models converges when the size of the model goes to infinity \cite{lynch1985probabilities}.

In literature, there are many logical frameworks developed that deal with boolean combinations of linear inequalities. In~\cite{DBLP:journals/iandc/FaginHM90}, the authors have developed a logical framework that expresses probabilities of combination of events in real numbers. In order to do that, they use a boolean combination of linear inequalities to reason. Although they have axiomatised the framework, the underlying model they use is a usual First Order logic model of domains of real numbers.

Counting models has been a major part of modal logic literature. A related work can be~\cite{torfah2014complexitycountingmodelslineartime}, where the authors consider Linear Temporal Logics. The paper gives decision versions of counting models: whether a formula has models upto an input bound. The paper proposes procedures for counting word and tree models, each one for when the bound is unary and one for when bound is in binary.. 


In~\cite{DBLP:journals/corr/abs-2606-23715}, the authors investigate zero-one law on various coloring and edge probability (proportion) conditions on Markov Logic Networks. They define ``regimes'' that describe the global structure of the coloring and edge assignment, for example, most vertices are colored using same, etc.  They showed results on whether zero-one laws and convergence laws hold depending on such regimes.






In~\cite{DBLP:conf/nips/Adam-DayIC23}, the authors proved 0-1 law for some graph neural networks but with strong assumptions. First, they focus on graph properties while we focus on node properties. But their stronger assumption is that they restrict to classifier that are non-splitting, meaning that the property is not on a limit case (see details in their paper). For instance $\# p \geq \# q$ is a limit case and its probability converges to 1/2 while  the probability $(1+\epsilon)\# p \geq \# q$ converges to 1 when the number of successors goes to infinity, when $\epsilon > 0$. Compared to them, we consider any $\Ksharp$-formula, without non-splitting assumption. However, we also consider truncated ReLU or eventually constant activation functions; setting which corresponds $\Ksharp$.

In~\cite{NEURIPS2024_e1cd1db6}, the authors investigate convergence laws of some classes of GNNs. They prove convergence by sampling large graphs from probabilistic graph models like Erd\H{o}s--R\'{e}nyi graphs. However, the GNN architecture they consider do not use summation in message passing, but mean.

%% file: section_background.tex

We present a modal logic of counting denoted by $\Ksharp$, that was introduced in \cite{DBLP:conf/ijcai/NunnSST24}, but also in \cite{DBLP:conf/fsttcs/BednarczykOPT21} and and \cite{DBLP:conf/icalp/BenediktLMT24} with a first-order syntax. Logic $\Ksharp$ has the same expressivity than graph neural networks with truncated ReLU as an activation function (\cite{DBLP:conf/ijcai/NunnSST24}, \cite{DBLP:conf/icalp/BenediktLMT24}).

\subsection{Kripke model}
We recall that a Kripke model is a directed graph like structures with each node having some propositional valuations over a countable set of propositions $\AP$. 
\begin{definition}
    A Kripke model is a tuple $\M = \ldiaarg{W, R, V}$ where 
         $W$ is a finite countable set of worlds,
       $R\subseteq W\times W$ is the accessible relation,
      and $V: W\rightarrow 2^\AP$ is the valuation function.
A pointed graph/model $\M, s$ is a Kripke model and a vertex $s$ in $W$.
\end{definition}

\subsection{Syntax}
Fix a  countable set of propositions $\AP$. The set of $\Ksharp$-formulas $\varphi$ is defined by the grammar:
\begin{center}
$\varphi, \psi, \dots := p\mid \neg\psi\mid \varphi\vee\psi\mid \xi\geq 0$
 \hspace{1cm} 
$\xi, \xi_1, \xi_2, \dots := c\mid
 1_\varphi\mid \Sharp\varphi\mid \xi_1 + \xi_2\mid c\times\xi_1$
\end{center}
where $p\in\AP$, $c$ is an integer constant. The non-terminal $\xi$ is for arithmetic expressions.
The construction $1_\varphi$ is an indicator representing whether $\varphi$ is true in the current state. The construction $\Sharp\varphi$ is the number of $\varphi$-successors.
We write $\AP(\phi)$ for the set of propositions appearing in $\phi$.

\begin{example}\label{example:ksharpformula}
Here is a well-formed $\Ksharp$-formula: $\phi = (\Sharp p\geq 3\Sharp(\Sharp r\geq 2))\wedge r$.
\end{example}

\subsection{Truth Conditions}
The truth conditions $\M, s \models \phi$ (formula $\phi$ is true at vertex $\M, s$) is defined as follows.
\begin{definition}
    Given a $\Ksharp$-formula $\varphi$ and a Kripke model $\M = \ldiaarg{W, R, V}$, $\M,s\models\varphi$ is defined  by induction on $\varphi$ by:
    \[
    \begin{array}{ccc}
        \begin{aligned}
            \M,s\models p & \mbox{ iff } p\in V(s)\\
            \M,s\models \neg\psi & \mbox{ iff } \M,s\not\models\psi\\
            \M,s\models \varphi\vee\psi & \mbox{ iff } \M,s\models\varphi\mbox{ or }\M,s\models\psi\\
            \M,s\models\xi\geq 0 & \mbox{ iff } [[\xi]]_{\M,s}\geq 0
        \end{aligned}
        &
        &
        \begin{aligned}
            [[c]]_{\M,s} & = c\\
            [[1_\varphi]]_{\M,s} & = 1 \text{ if $\M,s\models\varphi$, 0 otherwise} \\
            [[\Sharp\varphi]]_{\M,s} & = \sum_{t: (s,t)\in R}[[1_\varphi]]_{\M,t}\\
            [[\xi_1 + \xi_2]]_{\M,s} & = [[\xi_1]]_{\M,s} + [[\xi_2]]_{\M,s}\\
            [[c\times\xi]]_{\M,s} & = c\times[[\xi]]_{\M,s}
        \end{aligned}
    \end{array}
    \]
\end{definition}
\vspace*{5pt}

\begin{example}
Consider the following Kripke model 
\tikzstyle{vertex} = [draw,circle, minimum height=5mm, inner sep=0.5mm]
\begin{tikzpicture}[baseline=0mm, xscale=0.8, thick]
    \node[vertex, initial, initial text={}, label=right:{$r$}] (root) at (0, 1) {$root$};
    \node[vertex, label=below:{$p,r$}] (s3) at (-2, 0) {$s$};
    \node[vertex, label=below:{$p$}] (s5) at (2, 0) {$t$};
    \node[vertex, label=below:{$p, r$}] (s4) at (0, 0) {$u$};
    \node[vertex, label=below:{$r$}] (s6) at (4, 0) {$v$};
    \draw[->] (root) -- (s3);
    \draw[->] (root) -- (s5);
    \draw[->] (root) -- (s4);
    \draw[->] (s5) -- (s4);
    \draw[->] (s3) -- (s4);
    \draw[->] (s5) -- (s6);
\end{tikzpicture}
and the formula $\phi$ in~\Cref{example:ksharpformula} evaluated from the top node pointed with an arrow.
Formula $\phi$ holds true in this model from the root node. The root node has $r$ true in it. Now consider the $\Sharp p\geq 3\Sharp(\Sharp r\geq 2)$ part of the formula. Note that, $\Sharp p = 3$ (s,t and u). Also, consider the successor node where only $p$ is true (the node t). $t$ has two successors where $r$ holds (nodes $u$ and $v$). Hence $\Sharp r\geq 2$ holds in $t$. Note that $t$ is the only successor of root that holds $\Sharp \geq 2$. Hence $\Sharp(\Sharp r\geq 2) = 1$ at the root. Hence $\Sharp p\geq 3\Sharp(\Sharp r\geq 2)$ also holds true at root.
\end{example}

\begin{definition}[Modal Depth]
   The modal depth $md(\phi)$ of formula $\phi$ is defined 
    as:
    \begin{itemize}
        \item $md(p) = 0$, $md(\neg\psi) = md(\psi)$, $md(\varphi\vee\psi) = max(md(\phi), md(\psi))$, $md(\xi\geq 0) = md(\xi)$
        \item $md(c) = 0$, $md(1_\varphi) = md(\phi)$, $md(\Sharp\varphi) = md(\phi) + 1$, $md(\xi_1 + \xi_2) = max(md(\xi_1), md(\xi_2))$, $md(c\times\xi_1) = md(\xi_1)$
         
    \end{itemize}
\end{definition}

Intuitively, modal depth is the length of the longest nesting of the modality operator~$\Sharp$.

\begin{example}
We have
$
md((\Sharp(\Sharp p\geq\Sharp q)\geq 0)\wedge (\Sharp p\geq\Sharp q)) = 2
$. 
This is because $(\Sharp(\Sharp p\geq\Sharp q)\geq 0)$ has the maximum modal depth, which is $2$.
\end{example}

A subformula of a $\Ksharp$-formula is \emph{direct} when it does lie under the scope of exactly one $\Sharp$-modality. The concept of direct formula is important because we count the number of successors satisfying these direct subformulas to evaluate a formula in the current node. We write $\DSF(\phi)$ for the set of direct subformulas.

\begin{example}
\label{example:DSF}
In the formula $\phi = (\Sharp(\Sharp p\geq\Sharp q)\geq 0)\wedge (\Sharp r\geq\Sharp s)$, subformulas $\Sharp p\geq\Sharp q$,~$r$ and $s$ are direct while $p$ and $q$ are not. We have $\DSF(\phi) = \{\Sharp p\geq\Sharp q, r, s\}$.
\end{example}

Formally $\DSF(\phi)$ is defined as follows.

\begin{definition}[Direct Sub formulas]
    Given a \Ksharp formula $\varphi$, the set of \emph{direct subformulas}~$\DSF(\varphi)$ is defined by induction on $\varphi$ as:
    \newcommand{\DSFonexpr}{S'}
    
    \begin{minipage}{0.48\textwidth}
        \begin{itemize}
            \item $\DSF(p) = \emptyset$
            \item  $\DSF(\neg\psi) = \DSF(\psi)$
            \item $\DSF( \psi_1\vee\psi_2) = \DSF(\psi_1)\union \DSF(\psi_2)$
       	\item $\DSF( \xi\geq 0) = \DSFonexpr(\xi)$
        \end{itemize}
    \end{minipage}
    \hfill
    \begin{minipage}{0.48\textwidth}
        \begin{itemize}
            \item $\DSFonexpr(c) = \emptyset$
            \item $\DSFonexpr(1_{\varphi}) = \DSF(\varphi)$
            \item $\DSFonexpr(\Sharp\varphi) = \{\varphi\}$
            \item $\DSFonexpr(\xi_1 + \xi_2) = \DSFonexpr(\xi_1)\union \DSFonexpr(\xi_2)$
            \item $\DSFonexpr(c\times\xi) = \DSFonexpr(\xi)$
        \end{itemize}
    \end{minipage}
\end{definition}

\subsection{Link with other logics}

Note that standard modal logic can be expressed in $\Ksharp$: $\ldiamond \phi$ is a shorthand for $\# \phi \geq 1$ (there exists a $\phi$-successor), while $\lbox \phi$ is a shorthand for $\lnot \ldiamond \lnot \phi$. Graded modal logic can also be expressed in $\Ksharp$: $\ldiamond^{\geq k} \phi$ is a shorthand for $\# \phi \geq k$ (there exists at least $k$ $\phi$-successor).

\section{Counting models}
\label{section:countingproblem}

\newcommand\setofgraphssatisfying[2]{\mathcal G_{#1}^{#2}}
\newcommand\setoftreesatisfying[2]{\mathcal T_{#1}^{#2}}
\newcommand\modaldepthof[1]{md(#1)}

In this section, we define our counting problems. To do that, we suppose that $\AP$ is finite (otherwise there is already an infinite number of 1-vertex models!).

\subsection{Counting graphs}

We write $\mathcal G_{n}$ for the set of pointed graphs whose domain of size $n$. More precisely, we fix the domain to be $\set{s_1, \dots, s_n}$ and the pointed state is $s_1$, i.e. $$\mathcal G_{n} := \set{(\M, s_1), \suchthat \M = (W, R, V) \text{ with } W = \set{s_1, \dots, s_n}}.$$

We define 
$\setofgraphssatisfying{n}{\phi}$ to be the subset pointed graphs in $\mathcal G_{n}$ satisfying formula $\phi$. When not explicitly said otherwise, we suppose then that $\AP = \AP(\phi)$ i.e. that the set of propositions contains exactly the propositions appearing in $\phi$.

\begin{example}
The number of pointed graphs of size 2 on the proposition set $\set{p}$ is $2 \times 2 \times 4 \times 4$ (self-loop on 1, self-loop on 2, type of edges between 1 and 2, valuations). There are 4 pointed graphs of size 2 satisfying $p \land \ldiamond \lnot p \land \lbox \lnot p $:

   \begin{tikzpicture}[
           >=latex, 
           node distance=1cm, 
           thick,
           world/.style={circle, draw, minimum size=5mm, inner sep=0pt}
       ]
           \node[world] (w1) {$1$};
           \node[world, right of=w1] (w2) {$2$};
   
           \node[above=0cm of w1] {$p$};
           \node[above=0.15cm of w2] {};
   
           \draw[->] ([xshift=-5mm]w1.west) -- (w1.west);
   
           \draw[->] (w1) to node[above] {} (w2);
           
   
       \end{tikzpicture}
              \hfill
   \begin{tikzpicture}[
           >=latex, 
           node distance=1cm, 
           thick,
           world/.style={circle, draw, minimum size=5mm, inner sep=0pt}
       ]
           \node[world] (w1) {$1$};
           \node[world, right of=w1] (w2) {$2$};
   
           \node[above=0cm of w1] {$p$};
           \node[above=0.15cm of w2] {};
   
           \draw[->] ([xshift=-5mm]w1.west) -- (w1.west);
   
           \draw[->] (w1) to node[above] {} (w2);
           
            \draw[->] (w2) to [loop above] node {} (w2);
   
       \end{tikzpicture}
              \hfill
   \begin{tikzpicture}[
           >=latex, 
           node distance=1cm, 
           thick,
           world/.style={circle, draw, minimum size=5mm, inner sep=0pt}
       ]
           \node[world] (w1) {$1$};
           \node[world, right of=w1] (w2) {$2$};
   
           \node[above=0cm of w1] {$p$};
           \node[above=0.15cm of w2] {};
   
           \draw[->] ([xshift=-5mm]w1.west) -- (w1.west);
   
           \draw[->] (w1) to node[above] {} (w2);
            \draw[->] (w2) to node[above] {} (w1);
           
            \draw[->] (w2) to [loop above] node {} (w2);
   
       \end{tikzpicture}
       \hfill
     \begin{tikzpicture}[
             >=latex, 
             node distance=1cm, 
             thick,
             world/.style={circle, draw, minimum size=5mm, inner sep=0pt}
         ]
             \node[world] (w1) {$1$};
             \node[world, right of=w1] (w2) {$2$};
     
             \node[above=0cm of w1] {$p$};
             \node[above=0.15cm of w2] {};
     
             \draw[->] ([xshift=-5mm]w1.west) -- (w1.west);
     
             \draw[->] (w1) to node[above] {} (w2);
             \draw[->] (w2) to node[above] {} (w1);
     
         \end{tikzpicture}

         Said differently, when we consider $\AP = \set{p}$, $|\setofgraphssatisfying{n}{p \land \ldiamond \lnot p \land \lbox \lnot p}| = 4$.
\end{example}

 Then the proportion of pointed graphs of size $n$ satisfying $\phi$ is defined by:
$$
\mu_n (\phi) := \frac{|\setofgraphssatisfying{n}{\phi}|}{|\setofgraphssatisfying{n}{}|}.
$$
Note that the definition of $\mu_n(\phi)$ does not depend on $\AP$, as long as $\AP$ contains the propositions appearing in $\phi$. As presented in \cite{DBLP:journals/jsyml/Fagin76}, a logical framework $\mathcal{L}$ is defined to have zero-one law if for any formula $\phi\in\mathcal{L}$, 
        $\lim_{n\rightarrow\infty}\mu_n(\phi)\in\{0,1\}$. As discussed in the introduction, logic $\Ksharp$ does not have zero-one law.

Note that $|\setofgraphssatisfying{n}{}| = 2^{n^2 + |\AP| \times n}$ (choosing the set of edges and the valuations at each vertex). Thus,  computing $\mu_n(\phi)$ reduces to compute $|\setofgraphssatisfying{n}{\phi}|$.

\subsection{Counting trees}

Many modal logics have the tree-like model property: if a formula is satisfiable, then it is satisfiable in a pointed tree. That is why it is natural to count the number of trees satisfying a property wrt to the branching factor (i.e. outgoing arity).
We define $\mathcal T_{n, d}$ to be the set of pointed trees of depth $d$ and branching factor $n$. Similarly than for graphs, we define $\mathcal T_{n, d}^\phi$ to be the subset of pointed trees in $\mathcal T_{n, d}$ satisfying $\phi$.

\begin{example}
\newcommand\treeofdepthone[3]{
   \begin{tikzpicture}[
           >=latex, 
           node distance=1cm, 
           thick,
           world/.style={circle, draw, minimum size=5mm, inner sep=0pt}
       ]
           \node[world] (w1) {$1$};
           \node[world] (w2) at (-1, -1) {$2$};
   \node[world] (w3) at (1, -1) {$3$};
           \node[above=0cm of w1] {$#1$};
           \node[above=0mm of w2] {$#2$};
           \node[above=0mm of w3] {$#3$};
   
           \draw[->] ([xshift=-5mm]w1.west) -- (w1.west);
   
           \draw[->] (w1) to node[above] {} (w2);
           \draw[->] (w1) to node[above] {} (w3);
   
       \end{tikzpicture}}
       Here are the 4 pointed trees of depth 1 with branching factor 2 on the proposition set $\set{p}$ satisfying $\ldiamond p \land \lnot \lbox p$:
       
      \treeofdepthone{p}{p}{}
      \hfill
      \treeofdepthone{p}{p}{}
      \hfill
      \treeofdepthone{}{}{p}
      \hfill
      \treeofdepthone{}{}{p}
      
       Said differently, when we consider $\AP = \set{p}$, $|\setoftreesatisfying{n, 1}{\ldiamond p \land \lnot \lbox p}| = 4$.
\end{example}

\newcommand{\proportionoftrees}[2]{\mu^{\mathcal T}_{#1}(#2)}

We define the proportion of pointed trees of branching factor $n$ satisfying $\phi$ as follows:

$$\proportionoftrees n \phi :=
\frac{|\mathcal T_{n, \modaldepthof\phi}^{\phi}|}{|\mathcal T_{n, \modaldepthof\phi}^{}|}.
 $$

In the definition of $\proportionoftrees n \phi$, we consider trees of depth
 $\modaldepthof\phi	$. 
 Note that every vertices after depth $\modaldepthof\phi$ is not relevant when evaluating $\phi$.
 Again, for computing $\proportionoftrees n \phi$, it is sufficient to compute $\setoftreesatisfying{n, md(\phi)}{\phi}$ since $|\mathcal T_{n, \modaldepthof\phi}^{}| = 2^{\AP \times (2^{\modaldepthof{\phi}+1} - 1)}$ (choosing the valuation for each of the $2^{\modaldepthof{\phi}+1} - 1$ vertices).

%% file: section_counting.tex

In this section, we investigate the problem of computing $\proportionoftrees n \phi$  and $\mu_n(\phi)$, that is computing $|\setoftreesatisfying{n, md(\phi)}{\phi}|$ and $|\setofgraphssatisfying{n}{\phi}|$.

\subsection{Complexity of Counting Tree Models}
\newcommand{\brf}n 
We investigate on calculating the number of trees $|\setoftreesatisfying{n, md(\phi)}{\phi}|$ satisfying some $\Ksharp$-formula $\phi$.

\begin{definition}
The \CountingTree problem is defined as:
\begin{itemize}
    \item input: a $K^\#$-formula $\phi$, a number $n$ (branching factor)
    \item output: $|\setoftreesatisfying{n, md(\phi)}{\phi}|$, the number of pointed tree models $(\M, s)$, with branching factor $n$ and depth $md(\phi)$ such that $\M,s\models\phi$.
\end{itemize}
\end{definition}

In this subsection, we prove the following theorem.

\begin{restatable}{theorem}{CountTreeSharpPspace}
	\label{theorem:counttreessharpPSPACE}
    \CountingTree is in \Sharp \PSPACE, if $n$ is in binary (and also in unary).
\end{restatable}




\newcommand\modelsshallow{\models_L}
Consider a formula $\phi$. We now define the notion of \emph{local consistency} of $\phi$ at a given vertex in a tree. The only information we need is the propositional valuation $\nu$ and the number of successors in which each subformula in $\DSF(\varphi)$ is true. That last information is modeled by a function $\sem: \DSF(\phi) \rightarrow \ensN$ that assigns to each subformula $\psi \in \DSF(\varphi)$ by the number of successors in which $\psi$ hold. We now define formally local-consistency of $\phi$, denoted by $(\nu, \sem) \modelsshallow \phi$.

    \begin{definition}[Local Consistency $\models_L$]
        Consider a formula $\phi$, a propositional valuation $\nu$ over the propositions of $\phi$, and a function $\sem: \DSF(\phi)\rightarrow\mathbb{N}$. Formula $\phi$ is said to be \emph{local-consistent} with $(\nu, \sem)$, denoted by $(\nu,\sem)\models_L\phi$, if $\phi$ holds true after applying the following steps in order:
        \begin{enumerate}
            \item For each $\psi\in\DSF(\phi)$, replace each non-$\#$-nested occurrence $\Sharp\psi$ in $\phi$ by $\sem(\psi)$.
            \item Replace each proposition $p$ left by $\nu(p)$.
        \end{enumerate}
    \end{definition}


\newcommand{\exampleval}[2]{
\begin{array}{ll}
p & \mapsto #1 \\
\#r \geq 2 & \mapsto #2
\end{array}
}

\begin{example}
    Consider $\phi := q \land (\# p \geq 3 \# (\#r \geq 2))$. If $\nu[q] = 1$, $\sem(p) = 3$ and $\sem(r\geq 2) = 1$,  
    then $\phi$ is local consistent with $(\nu, \sem)$ since $q$ is true (since $v[q] = 1$), $\# p$ is replaced by 3, and $\# (\#r \geq 2)$ is replaced by 1, and $3 \geq 3\times 1$.
\end{example}

    A $\DSF(\varphi)$-valuation $V$ is a function $V: \DSF(\varphi) \rightarrow \set{0, 1}$. 
Furthermore, given a $\DSF(\varphi)$-valuation $V$, we define 
    $$
    descr(V) = \lbigand_{\psi\mid V[\psi]= 1} \psi \land \lbigand_{\psi \mid V[\psi]= 0} \lnot \psi
    $$.

For instance, $V(p) = 1$, $V(\# r \geq 2) = 0$, then $descr(V) = p \land \lnot( \# r \geq 2)$.

        



    We prove \Cref{theorem:counttreessharpPSPACE} using~\Cref{algo:guesstree} (\guesstree). We prove it by the following facts. 
    \begin{fact}
        The number of accepting executions of $\guesstree(\phi, n)$ is $|\setoftreesatisfying{n, md(\phi)}{\phi}|$.
    \end{fact}
    \begin{proof}
    \newcommand{\someacceptingexec}{\pi}
        Given a $\Ksharp$ formula $\phi$, the set $\AP$ of propositions appearing in $\phi$, and the branching factor $n$, we show that there is a one-to-one correspondence between non-rejecting executions of $\explorealltrees(\phi, \AP, n)$
       and trees in $\setoftreesatisfying{n, md(\phi)}{\phi}$. This sufficient because an accepting execution of $\guesstree(\phi, n)$ corresponds to a non-rejecting execution of $\explorealltrees(\phi, \AP, n)$.

       \newcommand{\setnonrejectingexecof}[1]{NonRejExec_{#1}}
         For that we first formally define, by induction on the modal depth of $\phi$, a function $f_\phi : \setnonrejectingexecof{\explorealltrees(\phi, \AP, n)}\rightarrow \setoftreesatisfying{n, md(\phi)}{\phi}$ where
        $\setnonrejectingexecof{\explorealltrees(\phi, \AP, n)}$ is the set of non-rejecting paths of $\explorealltrees(\phi, \AP, n)$.

      The base case is when $\phi$ is Boolean. By written $\someacceptingexec_\nu$ the execution that guessed valuation $\nu$ at Line 2, we set $f_\phi(\someacceptingexec_\nu)$ to the single pointed vertex tagged with the valuation $\nu$. In picture:
     
            $$
                f_{\phi}(P_\nu) = \begin{tikzpicture}[thick,baseline= 0.9cm, vertex/.style={draw,circle, minimum height=5mm, inner sep=0.5mm}]
                            \node[vertex, initial, initial text={}] (root) at (0, 1) {$\nu$};
                        \end{tikzpicture}
            $$

Induction case for a formula $\phi$. In that case, the execution starts by guessing a valuation $\nu$, guess a valuation $V_1$ over $\DSF(\phi)$, perform an execution $\someacceptingexec_1$ at line 11 with $descr(V_1)$, guess a valuation $V_2$ over $\DSF(\phi)$, perform an execution $\someacceptingexec_2$ at line 11 with $descr(V_2)$, etc. Said, differently, the execution is parameterized by $(\nu, V_1, \someacceptingexec_1, \dots, V_n, \someacceptingexec_n)$ and we write $\someacceptingexec_{(\nu, V_1, \someacceptingexec_1, \dots, V_n, \someacceptingexec_n)}$. The tree $f_{\phi}(\someacceptingexec_{(\nu, V_1, \someacceptingexec_1, \dots, V_n, \someacceptingexec_n)})$ is then a root tagged with valuation $\nu$ with children corresponding to the sub-executions $\someacceptingexec_1, \dots, \someacceptingexec_n$. The child for $\someacceptingexec_i$ is obtained via $f_{descr(V_i)}(\someacceptingexec_i)$.

            $$f_{\phi}(\someacceptingexec_{(\nu, V_1, \someacceptingexec_1, \dots, V_n, \someacceptingexec_n)}) = \begin{tikzpicture}[thick,baseline=0.5cm, vertex/.style={draw,circle, minimum height=5mm, inner sep=0.5mm}, 
            vertext/.style={
            }]
                            \node[vertex, initial, initial text={}] (root) at (0, 1) {$\nu$};
                            \node[vertext] (s1) at (-4, 0) {$f_{descr(V_1)}(\someacceptingexec_1)$};
                            \node[vertext] (s2) at (-1, 0) {$f_{descr(V_2)}(\someacceptingexec_2)$};
                            \node[] (s3) at (2, 0) {\ldots};
                            \node[vertext] (s3) at (4, 0) {$f_{descr(V_n)}(\someacceptingexec_n)$};
                            \foreach \x in {1, 2, 3} {
                                \draw[->] (root) -- (s\x);
                            }
                        \end{tikzpicture}$$

           Each accepting path corresponds to a satisfying model. In other words, we prove that function $f$ is a bijection. The proof goes by induction on the modal depth of $\phi$. The base case, when $\phi$ is Boolean, is straightforward. Now for inductive step, we proceed as follows.

        \noindent\textbf{Surjection.  } Consider a pointed tree model $\tau\in\setoftreesatisfying{n, md(\phi)}{\phi}$. 
        We construct an execution $\someacceptingexec$ as follows. First, we guess the valuation $\nu$ over $\AP(\phi)$ as in the root of $\tau$ at Line 6. Now consider child $i$ among the $n$ many children of the root.
         Note that $\DSF(\phi)$ are exactly the sub-formulas of $\phi$ that have exactly one less $\Sharp$-modal depth. Hence, in our execution, we guess the valuation over $\DSF(\phi)$ making exactly the formulas in $\DSF(\phi)$ true in the $i$-th child (Line 9 of the $i$-th iteration of for loop in Line 8). Hence at that $i$-child, $desc(V)$ holds. By induction hypothesis, there exists a non-rejecting execution $\someacceptingexec_i$ for the $i$-th call at Line 11. The obtained execution $\someacceptingexec$ is $\someacceptingexec_{(\nu, V_1, \someacceptingexec_1, \dots, V_n, \someacceptingexec_n)}$.
        
        \noindent\textbf{Injection.  } Take two non-rejecting executions $P_1\neq P_2$. Hence they differ on at least one of the nondeterministic choice. Just like the surjection proof, since the non-deterministic choices correspond to valuation of either propositions in the root or child nodes, they will correspond to different trees.

    \end{proof}
    \begin{fact}
        \explorealltrees runs in poly-space.
    \end{fact}
    \begin{proof}
        The number of nested calls is bounded by the $\#$-modal depth of the initial formula~$\phi$. Remember that the loop at Line 7 takes an exponential time in the size of $n$ (since $n$ is in binary), but the memory usage of each recursive call is cleaned after having been performed. The memory used at each call is polynomial in the size of the input: storing the valuation $\nu$ and $V$ is in $O(|\phi|)$ where $\phi$ is the initial formula. Also, for all $\psi \in \DSF(\phi)$, $\sem[\psi] \leq n$. So $\sem[\psi]$ can be represented with $O(\log n)$ bits which is of polynomial size to the input.
    \end{proof}

        


\begin{algorithm}
    \caption{\explorealltrees}
    \label{algo:explorealltrees}
    \KwIn{\Ksharp-formula $\phi$, $\AP = \AP(\phi)$ propositions of $\phi$, branching factor $n$ in binary} 
    \KwOut{\algoreject or hand over}
    \If{$\phi$ is Boolean}{
        Guess a valuation $\nu$ over $\AP$\;
        \algoif $\nu\not\models\phi$ \algothen \algoreject\;
    }
    \Else{
        Guess a valuation $\nu$ over $\AP$\;
        $\sem \gets [0]_{\DSF(\phi)}$\;
        \For{$i = 1\ldots n$}{
            Guess a valuation $V$ over $\DSF(\phi)$\;
            $\sem[\psi]\gets \sem[\psi] + 1$ for each $\psi\in\DSF(\phi)$ such that $V(\psi) = 1$\; 
            \explorealltrees($descr(V), \AP, n$)\;
        }
        \algoif $(\nu, \sem)\not\models_L\phi$ \algothen            \algoreject\;
    }
\end{algorithm}

\begin{algorithm}
    \caption{\guesstree}
    \label{algo:guesstree}
    \KwIn{\Ksharp-formula $\phi$, integer $n$ in binary} 
    \KwOut{\algoaccept if \explorealltrees hands over; \algoreject otherwise}
    \explorealltrees($\phi, n$)\;
    \algoaccept\;
\end{algorithm}


\begin{theorem}
Given $\phi$, deciding that $|\setoftreesatisfying{2}{\phi}| \geq 1$ is PSPACE-hard.
\end{theorem}

\begin{proof}
Reduction from the satisfiability problem of standard modal logic in a tree of branching~2 which is PSPACE-hard (this is folklore but stated formally in Lemma 14 of \cite{DBLP:conf/ijcai/SalzerST25}).
\end{proof}

Now counting is important. But usually we want to deal with proportions, for instance knowing that $\proportionoftrees{n}{\phi} \geq 1/2$. This problem reduces to know that $|\setoftreesatisfying{n, md(\phi)}{\phi}| \geq |\setoftreesatisfying{n,md(\phi)}{}|/2$. More generally, we aim at knowing whether $|\setoftreesatisfying{n, md(\phi)}{\phi}| \geq \kappa$. However, there is a trick in the representation of $\kappa$. Indeed, note that $\kappa$ represented in binary is not sufficient since $ |\setoftreesatisfying{n,md(\phi)}{}|/2$ is $2^{\AP \times (2^{md(\phi)+1}-1) - 1}$ which requires an exponential number of bits in $|\phi|$ to written down. However, that number is easily represented by $1\underbrace{0\dots \dots \dots \dots \dots \dots \dots \dots \dots \dots 0}_{\text{$\AP \times (2^{md(\phi)+1}-1) - 1$ numbers of 0}}$. And $\AP \times (2^{md(\phi)+1}-1) - 1$ can be written down with a polynomial number of bits. In other words, we represent $\kappa$ with a RLE (Run-length encoding) representation where each blocks of 0 and 1 are represented by the digit (0 or 1) and the number of occurrences of that digit.

\begin{theorem}
Given $\phi, n, \kappa$, deciding whether $|\setoftreesatisfying{2}{\phi}| \geq \kappa$ is in PSPACE, when $n$ is in binary, and $\kappa$ is in RLE compression representation.
 \end{theorem}

\begin{proof}
First by Ladner theorem, we have that $\# PSPACE = FPSPACE$ \cite{DBLP:journals/siamcomp/Ladner89}. It means there is a poly-space machine, given $\phi$, $n$, output on a write-only tape the number of trees with branching $n$ satisfying $\phi$. Note that this number can have an exponential number of bits.
Finally, we compare $\kappa$ to that output number on-the-fly (comparing numbers with RLE compression representation can be performed in time $O(N^2)$ where $N$ is the size of the RLE compression representation, see \cite{DBLP:conf/sequences/ApostolicoLS97}). 
\end{proof}

\subsection{Complexity of Counting Graph Models}
 Unlike the previous sections, which deals with $\Ksharp$ formulas being interpreted on pointed trees, we now generalise and look at pointed directed graphs. We call this problem as \CountingGraph problem. Formally,

\begin{definition}[$\CountingGraph$]
 The \CountingGraph problem is defined as
    \newline \emph{Input}: A $\Ksharp$ formula $\phi$, an integer $n$\\
    \emph{Output}: $|\setofgraphssatisfying{n}{\phi}|$, the number of pointed $n$-state Kripke models that satisfy $\phi$, where the set of propositions is $\AP = \AP(\phi)$.
\end{definition}

In the previous definition, remember that the domain of a Kripke model of size $n$ is~$\set{1, \dots, n}$ and that the designated point $s$ is always 1. We also suppose that the set of propositions is implicitly given by the formula $\phi$.




\begin{algorithm}
    \caption{\guesswithbias}
    \label{algo:guessgraph}
    \KwIn{\Ksharp-formula $\phi$, integer $n$ in unary} 
    \KwOut{Guess a pointed Kripke model $\M,s$ such that $|\M| = n$ and \algoaccept if $\M,s\models\phi$ or \algoreject otherwise}
    Non-deterministically guess a pointed graph $(\M,s)$ of $n$ vertices and on set of propositions $\AP = \AP(\phi)$\;
    \algoaccept if $\M,s\models\phi$ else 
    \algoreject\;
    \label{algorithm:guesswithbiais}
\end{algorithm}

\newcommand{\nbtotalgraphsexpression}
{2^{n^2 + n \times |\AP|}}

\begin{theorem}
\label{theorem:countgraphPsharp-c}
$\CountingGraph$ is \# \PTIME-complete when $n$ is given in unary.
\end{theorem}
\begin{proof}
    \noindent\newline\textbf{Upper Bound.} \Cref{algo:guessgraph} gives a nondeterministic algorithm $\guesswithbias$ taking as an input a formula $\phi$, an integer $n$ in unary, it guesses an $n$-state pointed Kripke model and checks if it satisfies $\phi$. 
    Each pointed graph $(\M, s)$ of $n$ vertices can be represented by its adjacency matrix and a valuation at each vertex, we suppose here that the designated point $s$ is always the first one. So there are $\nbtotalgraphsexpression$ possible models of size $n$.
    
          


    \begin{fact}
    $\guesswithbias(\phi,n)$ runs in poly-time in $|\phi|$, and $n$.
    \end{fact}
    \begin{proof}
        Let us analyze the algorithm.
        \begin{itemize}
            \item Line 1. This step consists in guessing a binary word of length $n^2 + n \times |\mathcal P|$ in time $n^2 + n \times |\mathcal P|$.
            \item Line 2. Checking $G, u \models \phi$ is in time $poly(|\phi|, n)$.
        \end{itemize}
    \end{proof}

    \begin{fact}
    [correctness]
    $|\{(\M,s)\mid\M,s\models\phi,|\M|=n\}|$ is exactly the number of accepting paths of $\guesswithbias(\phi,n)$.
    \end{fact}


    \begin{proof}
        One execution of $\guesswithbias(\phi, n)$ is guessing a distinct Kripke model $(G, u)$. So the number of accepting executions of 
        of $\guesswithbias(\phi, n)$ is
        $|\setofgraphssatisfying{n}{\phi}|$.
    \end{proof}

    \noindent\newline\textbf{Lower Bound. } Recall the problem of  \#SAT of propositional logic is: given a propositional formula $\phi$, compute the number of valuations that makes $\phi$ true. The problem  \#SAT is \#\PTIME-complete \cite{DBLP:journals/tcs/Valiant79}.
    
    We reduce $\emph{\#SAT}$ problem in poly-time to \CountingGraph as follows.  
    Given a formula~$\phi$, we build the \CountingGraph instance $(\phi, n=1)$. Note that the reduction is in poly-time. Finally, we have:  
    $
    \emph{\#SAT}(\phi) = k \mbox{ iff } |\setofgraphssatisfying{1}{\phi}| = k.
    $
\end{proof}

The same \Cref{algo:guessgraph} gives an exponential blow-up in the upper bound when~$n$ in binary.

\begin{theorem}
    $\CountingGraph(\phi,n)$ is in \#\EXPTIME when $n$ is given in binary.
\end{theorem}



\begin{theorem}
Given $\phi, n$, deciding whether $|\setofgraphssatisfying{\phi}{n}| \geq 1$ is \PSPACE-hard when $n$ is in binary.
\end{theorem}

\begin{proof}
By reduction from satisfiability problem of standard modal logic in a tree of branching~2 which is PSPACE-hard. Let $\phi$ be a formula of standard modal logic. Take $n$ to be the size of trees with branching $2$ of depth $d = md(\phi)$, 
i.e. $n = 2^{d} - 1$.
Take a formula $\tau$ that says that the model contains a tree of branching factor 2. Typically, $\tau$ is the formula $\lbigand_{i=0}^{d-1} \lbox^{i} (\lnot \ell_0 \land \dots \lnot \ell_{i-1} \land \ell_i \land \lnot \ell_{i+1} \land \dots \ell_{d-1}) \land  [ \ldiamond x_{i} \land \ldiamond \lnot x_{i} \land \lbigand_{j<i} (x_j \limply \lbox x_j) \land (\lnot x_j \limply \lbox x_j)$ where $\ell_0, \dots, \ell_{d-1}, x_0, \dots, x_{d-1}$ are fresh propositions (see \cite{DBLP:books/cu/BlackburnRV01}). The idea of $\tau$ is that there is a 2-branching over the Boolean value of $x_i$ while values of $x_j$ with $j < i$ are propagated. Leaves correspond to all valuations over $x_0, \dots, x_{d-1}$. The proposition $\ell_i$ means that we are at level $i$.
The reduction is $tr(\phi) := (\phi \land \tau, n)$. If $\phi$ is satisfiable in a tree of branching 2, then $\phi \land \tau$ is satisfiable in a tree of branching 2 which is a graph of size $n$. Conversely, if $\phi \land \tau$ is true in a graph of size $n$, then that graph must be a tree of branching 2.
\end{proof}

%% file: section_convergence.tex
\newcommand{\muT}{\mu^{\mathcal T}}
We now study the convergence of $\muT_n(\phi)$ when the branching factor $n$ goes to infinity.

\subsection{Counterexample for convergence}
\label{subsection:counterexampleconvergence}

In this subsection, we come up with an example of a $\Ksharp$ formula $\phi$ such that $\lim_{\brf\rightarrow\infty}\muT_n(\phi)$  does not exist.
%
Our candidate formula $\phi$ is $\# \phi_e \geq \# \phi_o$ where $\phi_e$ and $\phi_o$ are:
\begin{center}
$\phi_e := \# \top = 2\# p$
\hspace{1cm}
$\phi_o := \# \top = 2 \# p - 1$.
\end{center}
Intuitively, $\phi_e$ is a formula that asserts that the number of children to the current tree-node should be exactly twice the number of children where $p$ holds. In other words, number of children where $p$ holds is equal to the number of children where it does not.
Similarly, $\phi_o$ asserts that the number of children where $p$ holds is exactly $1$ less than the number of children where it does not. 

Unfortunately, $\proportionoftrees n {\# \phi_e \geq \# \phi_o}$ does not converge when $n$ goes to infinity: its value flips between 1 and numbers that are converging to 0; see Figure~\ref{figure:frenchambulance}.

\begin{figure}
    \centering
\begin{tikzpicture}
\begin{axis}[
    yscale=0.6,
    title={},
    xlabel={$n$},
    ylabel={$\muT_n(\#\phi_e \geq \#\phi_o)$},
    xmin=0, xmax=16,
    ymin=0, ymax=1.15,
    axis lines=left,
    grid=both,
    xtick={1,2,3,4,5,6,7,8,9,10,11,12,13,14,15},
    ytick={0, 0.25, 0.5, 0.75, 1.0},
]

\addplot+[
    ycomb,
    blue,
    thick,
    mark=*,
    mark options={fill=blue}
] coordinates {
    (1,  0.500)
    (2,  1.000)
    (3,  0.244)
    (4,  1.000)
    (5,  0.154)
    (6,  1.000)
    (7,  0.107)
    (8,  1.000)
    (9,  0.079)
    (10, 1.000)
    (11, 0.060)
    (12, 1.000)
    (13, 0.047)
    (14, 1.000)
    (15, 0.038)
};

\end{axis}
\end{tikzpicture}
 \caption{Values of $\muT_n(\# \phi_e \geq \# \phi_o)$ depending on the branching factor $n$.}
    \label{figure:frenchambulance}
\end{figure}

\begin{proposition}
$\lim_{\brf\rightarrow\infty}\proportionoftrees n {\# \phi_e \geq \# \phi_o}$ does not exist. 
\end{proposition}
\begin{proof}
    We only consider  trees with depth $2$. This is because the formula has modal depth of $2$, hence cannot evaluate nodes that lie deeper. We also treat the probability of a tree node holding $p$ true to be equally likely: a fair coin toss. Hence, each leaf node (depth $2$) represents one independent fair coin toss. Each node in the level above (depth $1$) represents an experiment that involves exactly $\brf$ many independent coin tosses. It is in this level where we evaluate the event space of the experiment based on $\phi_e$ or $\phi_o$. Lastly, in the root (depth $0$), we compare the count of the event spaces. Let us compute the probability of $(\# \phi_e \geq \# \phi_o)$ depending on the parity of the branching factor.
    \begin{itemize}
        \item {Even branching.} 
        When the branching is even, $\phi_o$ is false. There, $\# \phi_o$ is always 0. 
        Thus, formula $\# \phi_e \geq \# \phi_o$ is trivially true. 
        Thus we have, $\muT_{2n}(\# \phi_e \geq \# \phi_o) = 1$.
        
        \item {Odd branching. } 
        As $\#\phi_e$ is zero when the branching is odd, by definition $\muT_{2n-1}(\# \phi_e \geq \# \phi_o)$ is the probability that $\# \phi_o = 0$.
        Now, the probability that $\phi_o$ holds in a node is $\binom{2n-1}{n} / 2^{2n-1}$ (the probability of $n$ among $2n-1$ trials). So the probability of $\# \phi_o = 0$ is the probability of having $\phi_o$ false $2n-1$ times. So
        
        $$\muT_{2n-1}(\# \phi_e \geq \# \phi_o) := \left(1 - \binom{2n-1}{n} / 2^{2n-1}\right)^{2n-1}$$
        which tends to 0 when $n$ tends to $+\infty$ (see appendix \ref{appendix:section:limitoddfrenchambulance}).
    \end{itemize}

\end{proof}






\subsection{Alternative definition of the limit}

The following theorem gives a limit result for the probability of a formula $\phi$ to be true when the number of successors $n$ goes to infinity. Computing the limit of $\muT_n(\phi)$ directly causes technical issues that we will discuss after the following theorem.

The technical solution we propose to study a limit result works as follows. Instead of taking the limit of $\muT_n(\phi)$ when $n$ goes to infinity, we propose $n$ going to infinity at each level in the tree, and define a new quantity denoted by $\mu'_n(\phi)$. The base case is when $\phi$ is Boolean. In that case, $\mu'_n(\phi) := \muT_n(\phi)$ is constant and thus trivially $\lim_{n \rightarrow +\infty} \mu'_n(\phi)$ exists.
Now for the inductive case, we define $\mu'_n(\phi)$ as the limit when the branching factor at the root goes to infinity, provided that a direct subformula $\psi$ is true with probability $\lim_{n \rightarrow +\infty} \mu'_n(\psi)$ which exists by induction.

It turns out that we define $\mu'_n(\phi)$ \emph{inside} the following theorem because the existence $\lim_{n \rightarrow +\infty} \mu'_n(\psi)$ for the direct subformulas $\psi$ is actually the conclusion of the theorem: $\lim_{n \rightarrow +\infty}  \mu'_n(\phi)$ exists.

Observe that for all formulas $\phi$ of modal depth 1, $\mu'_n(\phi) = \muT_n(\phi)$.

\begin{definition}
We define $\mu'_n(\phi)$  by induction on $\phi$
    to be the probability that a rooted tree of the form 
    \begin{tikzpicture}[thick,every node/.style={draw, circle, fill=black, inner sep=0.5mm}, baseline = 3mm]
        \node (r) at (0, 0.5) {};
        \node (s1) at (-1, 0) {};
        \node (s2) at (-0.4, 0) {};
        \node[fill=white, draw=none] (sdot) at (0.4, 0) {$\dots$};
        \node (sn) at (1, 0) {};
        \draw[->] (r) -- (s1);
        \draw[->] (r) -- (s2);
        \draw[->] (r) -- (sn);
    \end{tikzpicture}
    such that for all $\psi\in\DSF(\phi)$, the probability that some successor satisfies $\psi$ is $\lim_{n \rightarrow +\infty} \mu'_n(\psi)$.
\end{definition}

The previous definition is well-formed because of the following theorem, and we kept the definition apart from the theorem for better clarity. If the limit $\lim_{n \rightarrow +\infty}  \mu'_n(\phi)$ exists, then it is written~$\mu'_\infty(\phi)$.

\begin{example}
Following the counterexample given in \Cref{subsection:counterexampleconvergence}, 
    note that $\mu'_\infty(\# \phi_e \geq \# \phi_o)$ exists. When the branching factor goes to infinity, the probability that $\phi_e$ holds at the root tends to $0$ (it becomes less and less probable to have exactly half of the successors to satisfy $p$). Similarly, the probability of $\phi_o$ also tends to 0. Therefore,   $\mu'_\infty(\phi_e) = \mu'_\infty(\phi_o) = 0$. Hence, $\mu'_n(\# \phi_e \geq \# \phi_o) = 1$ for any $n$. This implies $\mu'_\infty(\# \phi_e \geq \# \phi_o) = 1$.
\end{example}

\begin{theorem}\label{theorem:muprimeconverge}
   For all $\Ksharp$-formulas $\phi$, $\lim_{n \rightarrow +\infty}  \mu'_n(\phi)$ exists.
\end{theorem}

\begin{proof}

As per the syntax of $\Ksharp$, each formula is a boolean combination of literals and linear equations with $\#\psi$ as variables. Hence, we can write a $\Ksharp$ formula $\varphi$ as a DNF:
$$\phi \rightsquigarrow \bigvee_j (S_j \land \beta_j)$$
where the $S_j$'s are mutually exclusive systems (conjunctions) of inequalities and $\beta_j$ is a Boolean formula. Note that $\beta_j$ and $S_j$ are independent because $\beta_j$ is about the current vertex valuation and $S_j$ is about the successors.

\begin{example}
$\phi := (\# p \geq \# q) \lor (\# p \geq \# r) \land (p \lor \lnot q)$ is rewritten as
\begin{align*}
    [\underbrace{(\# p \geq \# q) \land (\# p < \# r)}_{S_1} \land \underbrace{(p \lor \lnot q)}_{\beta_1}] \lor \\
[\underbrace{(\# p < \# q) \land (\# p \geq \# r)}_{S_2} \land \underbrace{(p \lor \lnot q)}_{\beta_2}] \lor
\\
[\underbrace{(\# p \geq \# q) \land (\# p \geq \# r)}_{S_3} \land \underbrace{(p \lor \lnot q)}_{\beta_3}]
\end{align*}
\end{example}
For the rest of the proof, we assume that the given formula $\varphi$ is in DNF.

\noindent The proof proceeds by induction on the $\#$-modal depth. For $\#$-modal depth 0, the formula is Boolean. Therefore $\mu'_\brf(\phi)$ is constant and thus converges.

Now, we suppose the limits exist for a modal depth $d-1$. Consider the formula $\phi$ in DNF of modal depth $d$.


Consider the formulas $\psi_1, \psi_2, \ldots \psi_m\in\DSF(\varphi)$ such that they occur in $\varphi$ as $\#\psi_i$ for any $i\in[1..m]$. By induction,  $\mu'_\brf(\psi_i)$ converges to some $\mu'_\infty(\psi_i)$ when $\brf$ goes to $+\infty$.

We have to prove that $\mu '_\brf(\phi)$ converges to some $\mu'_\infty(\phi)$ when $\brf$ goes to $+\infty$.



Since $\varphi$ is in a DNF form, hence $\mu'_\brf(\phi) = \sum_j \mu'_\brf(S_j) \times \mu'_\brf(\beta_j)$. As $\beta_j$ are propositional, $\mu'_\brf(\beta_j)$ is constant. First we prove that if $ \lim_{\brf \rightarrow +\infty} \mu'_\brf(S_j)$ exists for all $j$, then  $ \lim_{\brf \rightarrow +\infty} \mu'_\brf(\phi)$ exists in the following fact. The rest of the proof becomes proving that $ \lim_{\brf \rightarrow +\infty} \mu'_\brf(S_j)$ exists for all $j$.

\begin{fact}
Let $S$ be a system of inequalities. Then $ \lim_{\brf \rightarrow +\infty} \mu'_\brf(S)$ exists.
\end{fact}
\begin{proof}
    For a given $\brf$, and $i=1..\brf$, and $\psi$ be a (direct) subformula, let $X_{\brf, i, \psi}$ be an indicator random variable such that
    $$
        X_{\brf, i, \psi} = 
        \begin{cases}
            1 \text{, if there is a subtree at index } i \text{ of n children satisfying }\psi\\
            0 \text{, otherwise}
        \end{cases}
    $$
    
    Therefore, $\proba{X_{\brf, i, \psi}= 1} = \lim_{\brf \rightarrow +\infty } \mu'_\brf(\psi_i) = \mu'_\infty(\psi_i)$.  
\end{proof}


Let $N_{\brf, \psi}$ be the number of successors satisfying $\psi$ among the at most $\brf$ successors. Formally, $N_{\brf, \psi} = \sum_{i=1..n} X_{\brf, i, \psi}$.
We define the $\DSF(\phi)$-vector $N_\brf := (\tuple{N_{\brf, \psi}}_{\psi \in \DSF(\phi)})^T$, that is a column vector of $\DSF(\varphi)$ dimension.
Formally, it is defined by $N_\brf := \sum_{i=1}X_{\brf, i, \cdot}$.

Let $k$ be the number of inequations in $S$.
The system $S$ for branching $n$ is of the form 
\begin{equation}
    AN_n \geq b
\end{equation}
where $A$ is $k\times \DSF(\phi)$-matrix; $b$ is a $k$ dimension column vector.
We use a version of Lindeberg-Feller Central Limit theorem~\cite{feller1957introduction} to prove the convergence/existence of the limit. Since we have vectors of sum of indicator random variable, we use the Cramer-Wold theorem and Multivariate CLT~\cite{billingsley1995probability}. For that purpose, we need the covariance matrix of every pair of indicator random variable.

Given a $\psi, \psi'\in\DSF(\varphi)$, the covariance between $X_{n,i,\psi}$ and $X_{n,j,\psi'}$ for any~$i,j$:

\begin{align}
\Sigma_{\psi, \psi'} & :=   \begin{cases}
                                \proba{X_{n,i,\psi}= 1 \text{ and } X_{n,j,\psi'}= 1} -  \mu'_\infty(\psi) \times  \mu'_\infty(\psi'), &\text{if } \psi\neq\psi'\\
                                \proba{X_{n,i,\psi}= 1}(\proba{X_{n,i,\psi}= 1}), &\text{if } \psi = \psi'
                            \end{cases}
 \notag\\
    & = \begin{cases}
            \mu'_\infty(\psi \land \psi') - \mu'_\infty(\psi) \times  \mu'_\infty(\psi'), &\text{if } \psi\neq\psi'\\
            \mu'_\infty(\psi)(1 - \mu'_\infty(\psi)), &\text{if } \psi = \psi'
        \end{cases}\label{equation:covariance}
\end{align}
Consider the covariance matrix $\Sigma = [\Sigma_{\psi,\psi'}]_{|\DSF(\varphi)|\times |\DSF(\varphi)|}$     

Since, the setup is a summation of indicator random variables, by applying Cramer-Wold theorem and then translating to Multivariate CLT~\cite{billingsley1995probability} we get:
\begin{align}
    Z_\brf := \frac{N_n - n \mu'}{\sqrt n} \xrightarrow{\brf \rightarrow +\infty}  \mathcal N_{\DSF(\varphi)}(0, \Sigma)\label{eq:CLT}
\end{align}







where $\Sigma$ is the covariance matrix whose entries are indexed by $\DSF(\phi)\times\DSF(\phi)$ and are given by \Cref{equation:covariance}. $\mathcal N_{\DSF(\varphi)}(0, \Sigma)$ denotes the Normal distribution of $|\DSF(\phi)|$ dimension vectors, whose mean is at $0$ and the variance along each of $|\DSF(\phi)|$ dimensions is $\Sigma$, the covariance matrix. 

Since we have our system of inequalities as $AN_\brf\geq b$, we rewrite system in~\Cref{eq:CLT} as:
\begin{align}
 A Z_\brf \geq \frac{b - n A \mu'}{\sqrt n} =  \underbrace{\frac{b}{\sqrt n} - \sqrt n A \mu'}_{c_n}\label{eq:CLTsystem}
\end{align}
Since as $\brf\rightarrow\infty$, the expression $\frac{b}{\sqrt{n}}\rightarrow 0$. Hence the validity of the~\Cref{eq:CLTsystem} depends only on $A\mu'$.

Recall that
$Z_n$, $\mu'$ are $\DSF(\phi)$-vectors, $A$ is $k\times\DSF(\phi)$-matrix; $b$ is a $k$-vector.
Also, the~\Cref{eq:CLT} becomes:
\begin{equation}
    A Z_\brf \xrightarrow{\brf \rightarrow +\infty} \mathcal N_k(0, A\Sigma A^t)\label{eq:CLTdist}
\end{equation}

Since each row in $A$ are the co-efficients of one inequality in the system, depending on how rows in $A\mu'$ (row $i$ denoted by $[A\mu']_i$) behaves, the convergence of the coordinates of $c_n$ can be one of following cases:

\begin{enumerate}
\item If at least for one $i$, $[A \mu']_i<0$, then $c_n$ converges to $+\infty$. So system $S$ is unsatisfiable at the limit. Therefore $\lim_{n\rightarrow\infty}\mu'_n(S) = 0$.

\item If for all $i$, $[A \mu']_i > 0$,  $c_n$ converges to $-\infty$. Hence $AZ_\brf$ grows faster than the RHS in~\Cref{eq:CLTsystem} as $n\rightarrow\infty$ which implies $\lim_{n\rightarrow\infty}\mu'_n(S) = 1$

\item If for every $i$, $[A \mu']_i = 0$, then $c_n$ converges to 0. We have $\mu'_\brf(S) = \proba{AZ_\brf\geq \frac{b}{\sqrt{n}}}$. By~\Cref{eq:CLTdist}, $AZ_\brf\convergesn\mathbf{Z}\sim \mathcal N_k(0, A\Sigma A^t)$ and since $\lim_{n\rightarrow\infty}\frac{b}{\sqrt{n}} = 0$, by Portmanteau theorem~\cite{billingsley1995probability}:
\begin{align*}
    \mu'_\infty(S) & = \lim_{n\rightarrow\infty}\proba{AN_\brf\geq b}\\
            & = \lim_{n\rightarrow\infty}\proba{AZ_\brf\geq\frac{b}{\sqrt{n}}}\\
            & = \proba{A\mathbf{Z}\geq 0}\mbox{, applying Portmanteau theorem}
\end{align*}

\end{enumerate}


\end{proof}

\begin{example}
    With system $\# p \geq 5 \land \# q \geq 2$, we have matrix $A = \begin{pmatrix}
    1 & 0 \\
    0 & 1
    \end{pmatrix}$ and
    $\mu' = \begin{pmatrix}
    1/2 \\ 1/2 
    \end{pmatrix}$. 
    Thus we are in case 2 and all the constraints can be ignored. We are left with a matrix of "relevant" constraints of size $0 \times 0$. Thus $\mu(\phi) = 1$.
\end{example}

\begin{example}
    With system $\# p - \# q \geq 5 \land 2 \# p - 3 \# r \geq 2$, we have matrix $A = \begin{pmatrix}
    1 & -1 \\
    2 & 3
    \end{pmatrix}$ and
    $\mu' = \begin{pmatrix}
    1/2 \\ 1/2 
    \end{pmatrix}$, thus $A\mu' = \begin{pmatrix}
    0 \\ -1/2 
    \end{pmatrix}$.
    As the second coordinate of $A\mu'$ is negative, we are in point 1 and $\mu(\phi) = 0$.
\end{example}

\begin{example}
    With system $\# p - \# q \geq 5 \land 3 \# p - 2 \# r \geq 2$, we have matrix $A = \begin{pmatrix}
    1 & -1 \\
    3 & 2
    \end{pmatrix}$ and
    $\mu' = \begin{pmatrix}
    1/2 \\ 1/2 
    \end{pmatrix}$, thus $A\mu' = \begin{pmatrix}
    0 \\ 1/2 
    \end{pmatrix}$.
    As the second coordinate of $A\mu'$ is positive, so we can ignore the second constraint.
    
    At the limit, system $\# p - \# q \geq 5 \land 3 \# p - 2 \# r \geq 2$ has the same probability to be true than system $\# p - \# q \geq 5$. The limit probability is $P(Z_1 - Z_2 \geq 0)$ where $Z$ is a 2-dimensional normal distribution. Since $Z_1-Z_2$ is a standard normal distribution, $P(Z_1 - Z_2 \geq 0) = 1/2$.
\end{example}

\begin{example}
    With system $\# p - \# q \geq 5 \land  \# p -  \# r \geq 2$, we have matrix $A = \begin{pmatrix}
    1 & -1 & 0 \\
    1 & 0 & -1
    \end{pmatrix}$ and
    $\mu' = \begin{pmatrix}
    1/2 \\ 1/2 \\ 1/2
    \end{pmatrix}$, thus $A\mu' = \begin{pmatrix}
    0 \\ 0
    \end{pmatrix}$.
   
\end{example}

\subsection{Probability that are reached}

\newcommand\muTlimits{\muT_\infty}

In this section, we study the value of the limit probability $\muTlimits(\phi) = \lim_{n \rightarrow +\infty} \muT_n(\phi)$ that exists and can be reached by some $\Ksharp$-formula $\phi$.
Let us recall Sheppard’s Formula.

\begin{lemma}
[Sheppard’s Formula \cite{DBLP:journals/corr/abs-2105-10386}]. Let $Z_1, Z_2$ be standard 
1-dimension Gaussian random variables
with correlation $E[Z_1 Z_2] = \rho \in [-1, 1].$
Then
$$Pr[Z_1 \leq 0, Z_2 \leq 0] = \frac 1 2 - \frac 1 2 \frac {\arccos \rho} \pi.$$
\end{lemma}

\begin{proposition}
For all numbers $x \in [0, 1/2]$, there is a sequence of a $\Ksharp$-formula $\phi_0, \phi_1, \dots$ that are systems with two inequalities such that $\muTlimits(\phi_n) \rightarrow x$.
\end{proposition}

\begin{proof}
    Given a matrix $A$ with two rows $A_1$ and $A_2$ let us compute $P(AZ \geq 0)$. According to Sheppard's formula, $P(AZ \geq 0) = \frac 1 2 - \frac 1 2 \frac {\arccos \rho} \pi$ where $\rho = \frac {A_1 \cdot A_2}{||A_1|| ||A_2||}$.
    We can make approximate all values for $\rho$ between $[-1, 1]$ with arbitrary big matrices $A$. Since $\arccos\rho\in [0,\pi]$, all values of $[0, 1/2]$ can be approximated by $P(AZ \geq 0)$. 
\end{proof}

The probability $\muTlimits(\phi)$ over $K\Sharp$-formulas $\phi$ can get arbitrarily close to any number in $[0,1]$ as stressed by the following proposition.
    
\begin{proposition}
For all $x \in [0, 1]$, for all $\epsilon > 0$, there is a $\Ksharp$-formula $\phi$ such that $|\muTlimits(\phi) - x| < \epsilon$.
\end{proposition}

\begin{proof}
    We prove the statement using the fact that the set of dyadic rationals $\{\frac{k}{2^m}\mid k\in\mathbb{N}\union\{0\}, m\in\mathbb{N}, 0\leq k\leq 2^m\}$ is dense in $[0,1]$. So for all $x \in [0, 1]$, for all $\epsilon > 0$, there is $m, k$ such that $|x - \frac k {2^m}| < \epsilon$. We now give a formula $\phi$ with $\muTlimits(\phi) = \frac k {2^m}$. Consider $$\phi := \lbigor_{j=1..k} \lbigand_{i=1..m} (\# p_{i,j} \leq \# q_{i,j})$$
    where
    $p_{i,j}$ and $q_{i, j}$ are distinct propositional variables. Since each propositions are independent and $\muTlimits(\# p_{i,j}\geq \# q_{i,j}) = 1/2$, the ANDs product the probabilities and the ORs adds the probabilities up.
\end{proof}

\begin{conjecture}
$\lim_{n \rightarrow +\infty} \mu'_n(\phi)$ is equal to some adherence point of $(\muT_n(\phi))_{n \in \mathbb N}$.
\end{conjecture}

%% file: section_conclusion.tex

This work opens the path of studying convergence on general graphs and not only trees, as done in \Cref{section:convergence}.
We finish our work by listing some open problems.

\begin{itemize}
\item How to decide when $\mu_n(\phi)$ converges? Is it decidable? What is the complexity of that decision problem? Are there some approximation algorithms? These questions are relevant to solve verification problems of GNNs for large graphs.
In fact, we conjecture that a fragment, we call majority logic, with a \emph{boolean bias} condition can have 0-1 law. (for more details, see Appendix)

\item How to put proportion information in the logic and construct a proof system for reasoning about proportions of models? The answer to this research question may provide a tool for giving justifications of some verification verdict.
  \item How to compute conditioned proportion? For instance how to compute the proportion of $\phi$-models among the $\psi$-models? This is relevant for more involved verification tasks. 
\end{itemize}

%% file: appendix.tex
\section*{Convergence to 0 of $\mu_{2n-1}(\# \phi_e \geq \# \phi_o)$}
\label{appendix:section:limitoddfrenchambulance}


\begin{theorem}
    \[
        \lim_{n \to \infty} \left( 1 - \frac{\binom{2n-1}{n}}{2^{2n-1}} \right)^{2n-1} = 0
    \]
\end{theorem}
\begin{proof}

Note that:

\begin{align*}
    \binom{2n-1}{n} &= \frac{(2n - 1)!}{n!\times(n-1)!}\\
                    &= \frac{n}{2n}\frac{(2n)!}{n!\times n!}\\
                    &= \frac{1}{2} \binom{2n}{n}
\end{align*}

Hence the statement expression becomes

$$
\left(1 - \frac{\binom{2n-1}{n}}{2^{2n-1}}\right)^{2n-1} = \left(1 - \frac{\binom{2n}{n}}{2^{2n}}\right)^{2n-1}
$$

Now we approximate the expression asymptotically. Note that by Stirling's approximation~\cite{DBLP:books/aw/GKP1994}, $\lim_{n\to\infty} n!\sim \sqrt{2\pi n}(\frac{n}{e})^n$. Hence,
\[
\binom{2n}{n} = \frac{(2n)!}{(n!)^2} \sim \frac{\sqrt{4\pi n} \left(\frac{2n}{e}\right)^{2n}}{\left( \sqrt{2\pi n}\left(\frac{n}{e}\right)^n \right)^2} = \frac{2^{2n}}{\sqrt{\pi n}}
\]

Therefore, as $n \to \infty$:
\[
\frac{\binom{2n}{n}}{2^{2n}} = \frac{1}{\sqrt{\pi n}} 
\]


As $n\to\infty$, let 
$$L_n = \left(1 - \frac{\binom{2n}{n}}{2^{2n}}\right)^{2n-1} = \left(1 - \frac{1}{\sqrt{\pi n}} \right)^{2n-1}$$

Taking logarithms on both sides:
$$
\ln L_n = (2n-1)\ln\left(1 - \frac{1}{\sqrt{\pi n}}\right)
$$

As $n\to\infty$, $\frac{1}{\sqrt{\pi n}}\to 0$. Also, as $x\to 0$, by Taylor expansion,
$$
    \ln(1 - x) = -x + O(x^2)
$$
Hence,
\begin{align*}
    \lim_{n\to\infty}\ln L_n &= \lim_{n\to\infty} (2n-1)\ln\left(1 - \frac{1}{\sqrt{\pi n}}\right)\\
                             &= \lim_{n\to\infty} (2n-1)\left(-\frac{1}{\sqrt{\pi n}} + O\left(\frac{1}{n}\right)\right)\\
                             &\sim \lim_{n\to\infty} -\frac{2n-1}{\sqrt{\pi n}}\\
                             &\sim \lim_{n\to\infty} -\frac{2\sqrt{n}}{\sqrt{\pi}}\\
                             &\sim -\infty
\end{align*}
Therefore
$$
\lim_{n\to\infty} L_n = e^{-\infty} = 0
$$



\end{proof}



\section{Computations}

\subsection{Computation of proportion of pointed graphs for $\# p \geq \# q$}

The number of graphs of size $n$ is $4^n \times 2^{n^2}$.

The number of configurations for the successors of a point to satisfy $\#p \geq \#q$ is

$$\sum_{S \subseteq \set{1, \dots, n}} N_\geq(\# S)$$

where $N_\geq(m)$ is the number of ways to put subsets $P, Q$ in a set of size $m$ and with $\# P \geq \# Q$. \subsection{Computation of proportion of pointed graphs for $\# p \geq \# q$}

$$N_\geq(m) = \frac{4^m + \binom{2m}m} 2.$$

The number of configurations is:

$$\sum_{m=0}^n \binom n m \frac{4^m + \binom{2m}m} 2$$

We now also have to consider the choices elsewhere. So the number of models satisfying $\#p \geq \#q$ is:

$$\sum_{m=0}^n \binom n m \frac{4^m + \binom{2m}m} 2 \times 4^{n-m} \times 2^{(n-1)n}$$

So the probability that $\#p \geq \#q$ holds in a pointed graph of size $n$ is:

$$\frac{\sum_{m=0}^n \binom n m \frac{4^m + \binom{2m}m} 2 \times 4^{n-m} \times 2^{(n-1)n}}{4^n \times 2^{n^2}}.$$

$$\sum_{m=0}^n \binom n m \frac{4^m + \binom{2m}m} {4^m \times 2^{n+1}}.$$

\subsection{Computation of proportion of pointed trees for $\# p \geq \# q$}

We compute the probability of satisfying $\# p\geq\# q$ in tree model from the root with number of branching at most $B$. 

Also, since truth depends on modal depth, we do not care about depth.

Number of trees of at most $B$ branching with two propositions: 
$$
4\times\sum_{b = 0}^B (4^b)
$$

From the current node, we just count that given number of branches $b$, how many ways can their be valuations such that number of positive $\# p$ is at least $\# q$ just like above:
$$
4\times\sum_{b = 0}^B\sum_{P\subseteq [0, b]}(\sum_{Q\subseteq [0,b], |P|\geq |Q|} 1)
$$

By above we have 
$$
\sum_{P\subseteq [0, b]}(\sum_{Q\subseteq [0,b], |P|\geq |Q|} 1) = \frac{4^b + \binom{2b}{b}}{2}
$$

Hence
$$
\mu_B(\# p\geq\# q) = \frac{4\times\sum_{b = 0}^B\frac{4^b + \binom{2b}{b}}{2}}{4\times\sum_{b = 0}^B (4^b)}
$$

\section{A Condition for 0-1 Law for Majority Logic}

In this section, we introduce a fragment of $\Ksharp$ that we call
Majority Logic. The restriction is as follows. The inequations we allow are of the form $\# \phi \geq \# \psi$. Majority logic does not have 0-1 law (e.g. $\muT_n(\# p \geq \# q) \convergesn 1/2$). But we introduce a constraint, called \emph{full Boolean bias}, that gives majority logic 0-1 law back.

\subsection{Boolean Bias Condition}
The idea is to resolve all the Boolean bias.
%
Given $\proba{p_1}, \dots, \proba{p_n}$, we compute the probability of $\proba{\chi}$ for a Boolean formula $\chi$ by assuming that the truth values of $p_1, \dots, p_n$ are independent.

\begin{definition}
    $\proba{p_1}, \dots, \proba{p_n}$ has \emph{full Boolean bias} iff for all Boolean formulas $\chi_1, \chi_2$, 
    if $\proba{\chi_1} = \proba{\chi_2}$ then
    $\chi_1 \equiv \chi_2$.
\end{definition}

\begin{example}[no full Boolean bias]
    If $\proba p = 1/3$ and $\proba q = 1/4$, then there is a no full Boolean bias because:
    \begin{itemize}
        \item  $\proba {p\land q} = 1/3 \times 1/4$
        \item  $\proba {p\land \lnot q} = 1/3 \times 3/4$
        \item  $\proba {\lnot p\land \lnot q} = 2/3 \times 3/4 = 1/2$ (no bias)
        \item  $\proba {p\lor q} = 1 - 1/2 = 1/2$
    \end{itemize}
\end{example}

\begin{example}[full Boolean bias]
    If $\proba p = 1/3$ and $\proba q = 1/5$, then there is a full Boolean bias because:
    \begin{itemize}
        \item  $\proba {p\land q} = \proba p \times \proba q = 1/3 \times 1/5$
        \item  $\proba {p\land \lnot q} = 1/3 \times 3/5$
        \item  $\proba {\lnot p\land \lnot q} = 2/3 \times 3/5 = 2/5$
        \item  $\proba {p\lor q} = 1 - 2/5 = 3/5$
    \end{itemize}
\end{example}

The following is our conjecture:
\begin{conjecture}
Under the condition \emph{full Boolean biais}, for all $\phi$ of majority logic, 
$$\lim_{n \rightarrow +\infty}\mu_n(\phi) \in \set{0, 1}.$$
\end{conjecture}